\documentclass[hyperref]{article}

\usepackage{microtype}
\usepackage{graphicx}
\usepackage{subfigure}
\usepackage{booktabs} 

\usepackage{hyperref}

\usepackage[accepted]{icml2022}

\usepackage{amsmath}
\usepackage{amssymb}
\usepackage{mathtools}
\usepackage{amsthm}

\usepackage[capitalize,noabbrev]{cleveref}

\theoremstyle{plain}
\newtheorem{theorem}{Theorem}[section]
\newtheorem{example}{Example}[section]
\newtheorem{proposition}[theorem]{Proposition}
\newtheorem{lemma}[theorem]{Lemma}

\newtheorem{definition}[theorem]{Definition}
\newtheorem{assumption}[theorem]{Assumption}
\theoremstyle{remark}

\newcommand{\bx}{\mathbf{x}}
\newcommand{\by}{\mathbf{y}}
\newcommand{\bz}{\mathbf{z}}
\newcommand{\bX}{\mathbf{X}}
\newcommand{\bZ}{\mathbf{Z}}
\newcommand{\bbR}{\mathbb{R}}
\newcommand{\bbN}{\mathbb{N}}
\newcommand{\bbE}{\mathbb{E}}
\newcommand{\calL}{\mathcal{L}}
\usepackage[textsize=tiny]{todonotes}

\DeclareMathOperator*{\argmin}{arg\,min}

\icmltitlerunning{Reconstructing Training Data From Bayesian Posteriors and Trained Models}

\begin{document}

\twocolumn[
\icmltitle{On Reconstructing Training Data From Bayesian Posteriors and Trained Models}



\icmlsetsymbol{equal}{*}

\begin{icmlauthorlist}
\icmlauthor{George Wynne}{yyy}
\end{icmlauthorlist}

\icmlaffiliation{yyy}{Alan Turing Institute}

\icmlcorrespondingauthor{}{gwynne23@turing.ac.uk}

\icmlkeywords{Machine Learning, ICML}

\vskip 0.3in
]


\printAffiliationsAndNotice{}  

\begin{abstract}
Publicly releasing the specification of a model with its trained parameters means an adversary can attempt to reconstruct information about the training data via training data reconstruction attacks, a major vulnerability of modern machine learning methods. This paper makes three primary contributions: establishing a mathematical framework to express the problem, characterising the features of the training data that are vulnerable via a maximum mean discrepancy equivalance and outlining a score matching framework for reconstructing data in both Bayesian and non-Bayesian models, the former is a first in the literature. 
\end{abstract}

\section{Introduction}\label{sec:introduction}
It is common for trained models to be released to the public where both the architecture is known as well as the trained parameters, or in the case of Bayesian models the likelihood and priors are known and posterior samples are released. This could be for benchmarking purposes or for general contributions towards open source software. Examples include simple benchmark neural networks such as Resnet \citep{He_2016_CVPR}, the \texttt{posteriordb} database \citep{posteriordb} and large open source diffusion models \citep{Rombach_2022_CVPR}.

It is known that such public trained models can leak various notions of information about the training sets they were trained on. Examples include membership inference \citep{membership_inference}, attribute inference \citep{attribute_inference} and training data reconstruction attacks (DRA) \citep{Haim2022,buzaglo2023deconstructing,Loo2023,Runkel2024,Guo2022_bounding,Hayes203_dpsgd,Kaissis2023_hypothesis}. The latter is the focus of this paper and it is the most potent attack since it attempts to fully reconstruct the training data used for the model.

The vast majority of work regarding training data reconstruction attacks has focused on neural networks for classifying images where the aim has been perfect reconstruction of the images \citep{Haim2022,Guo2022_bounding,Loo2023,Zhu2019_leakage}. Other sorts of data and other models, for instance regression tasks or Bayesian models have so far received minimal attention.

Furthermore, although there is a growing volume of work on producing more elaborate and cunning data reconstruction attacks, the question of when it is even possible to recover training data has not been fully interrogated. Indeed, failure to reconstruct training data from a given model is viewed as a \textit{failure of the reconstruction method}, rather than a \textit{property of the model architecture itself}. 

Both these issues - the focus on a limited model class and the wider question of when training data reconstruction attacks should even work - can be encapsulated succinctly by a thought experiment involving an almost trivially simple model. Suppose we were performing classical linear regression with training data set $\{(0,0),(1,1)\}$ resulting in the model $f(x) = x$. If this model is released, along with its training procedure, then the adversary cannot possibly recover the exact training data because other datasets will result in the same trained model, for example $\{(1,1),(2,2)\}$. This shows that if the \textit{sufficient statistic} of the model is not equal to the original training data set then one cannot hope to achieve perfect reconstruction. But what about the inbetween case where more information about the training data is used in the model, can more information about the training data be reconstructed?

This paper aims to answer this question in a quantitative manner and presents the following main contributions. 
\begin{itemize}
	\item A statistical formulation of the training data reconstruction problem in terms of empirical measures of the training data, facilitating the use of statistical divergences to understand how vulnerable a training data set is.
	
	\item A novel attack is devised for Bayesian models, so far absent in the literature, revealing vulnerabilities in such models which were not known.
	
	\item A theorem which characterises the features of training data that can be recovered in a data reconstruction attack and how these features depend on the features of the data used in the model.
	
	\item A simple method to link the reconstructions methods that exist in the literature for non-Bayesian models to the proposed reconstruction method for Bayesian models, to show the perspective presented in this paper is general. 
\end{itemize}

The rest of the paper is structured as follows. Section \ref{sec:goals} will outline the threat model and the mathematical framework of data reconstruction. Section \ref{sec:Bayesian} covers the case of Bayesian models, a novelty in the DRA area, with Subsection \ref{subsec:bayesian_recovery} proposing a novel method for reconstructing data from Bayesian posteriors and Subsection \ref{subsec:bayesian_char} provides a result characterising what features of data can be recovered in the Bayesian case. Analogously, Section \ref{sec:non_bayesian} covers non-Bayesian models with Subsection \ref{subsec:non_bayesian_recovery} detailing how the reconstruction method in Subsection \ref{subsec:bayesian_recovery} naturally adapts to the non-Bayesian case and Subsection \ref{subsec:non_bayesian_char} provides a result characterising what features of data can be recovered in the non-Bayesian case. Section \ref{sec:numerics} provides a detailed numerical example and Section \ref{sec:conclusion} provides concluding remarks and future research directions. 

\subsection{Existing Work}
There is a fast maturing literature on adversarial machine learning and the variety of attacks that can be performed on models under a variety of threat models. For a broad overview consult \citet{Ponomareva2023_DPfy,Dwork2017}.

The focus of this paper shall be data reconstruction attacks (DRA), also known as training data recovery attacks. The majority of techniques are optimisation centric, meaning that the reconstruction is framed as an optimisation problem with a corresponding loss which is minimised. The most common loss function for the reconstruction task is formed by taking the norm of the gradient of the training loss function with respect to the model parameters \citep{Haim2022,buzaglo2023deconstructing,Loo2023}, see Section \ref{sec:non_bayesian}. The idea is that is this is zero then the training loss function evaluated at the released model parameters and reconstructed data will be zero, indicating that the reconstructed data fits the model well. 

Typically, the literature focuses on a particular class of models. Most commonly it is neural networks used for image classification, since here it is visually easy to inspect when a DRA has succeeded \citep{Loo2023,Haim2022,Guo2022_bounding,Runkel2024}.

There are a variety of assumptions made about the threat model. Papers which assume that the adversary has access to the same distribution as the training data include \citet{Kaissis2023_hypothesis,Balle2022_informed,Hayes203_dpsgd}. Our paper assumes that the adversary does not have such information and other papers which also assume the adversary does not have such information include \citet{Haim2022,Loo2023,buzaglo2023deconstructing,Runkel2024}.

Though not directly related, but still heavily lying in the theme of analysing the relationship between datasets and trained models, are the topic of coresets and data distillation, for reviews see \citet{winter2023_future,sachdeva_2023_distillation}. These areas attempt to find datasets which are typically smaller than the training data but still produce the same trained model. Therefore, minus the requirement to produce a smaller dataset, coreset and data distillation algorithms actually bare a strong resemblance to DRA algorithms which aim to discover the data which produces the trained model. The key difference is what is known to the user or adversary. Coreset algorithms attempt to find a smaller dataset to produce the same trained model \textit{without} training the model where as in DRA one starts with the trained model and attempts to find the dataset which trained it. 

Finally, a topic related to DRA is sufficient statistics \citep{Casella2024} - the minimum amount of information, the sufficient statistic, to characterise a model. This is intrinsically related to DRA attacks as the attacker aims to find the data which characterises the trained model. For the Bayesian case it is the notion of Bayesian sufficiency which is relevant to DRA \citep{Blackwell1982,Bernardo1994}.

\subsection{Preliminaries}
This subsection shall establish  notation and assumptions which permeate the rest of the paper. 

\textbf{General notation:} Bold capital letters will denote datasets with lowercase letters denoting individual data samples e.g. $\mathbf{X} = \{\mathbf{x}_{n}\}_{n=1}^{N}$ and non-bold lowercase will denote components of data samples e.g. $\mathbf{x}_{n} = (x_{n},y_{n})$. Capital $M,N\in\bbN$ will denote data set sizes, $\bX = \{\bx_{n}\}_{n=1}^{N}$ a training data set and $\bZ = \{\bz_{m}\}_{m=1}^{M}$ pseudo-data set to be optimised in an attempt to reconstruct $\bX$. Scalar weights for the pseudo-data are denoted $w = (w_{1},\ldots,w_{M})\in\bbR^{M}$. Capital letters will denote measures e.g. $P$ and the un-normalised empirical distribution of the training data is denoted $P_{\bX} = \sum_{n=1}^{N}\delta_{\bx_{n}}$ and the un-normalised weighted distribution based on the pseudo-data $P_{w,\bZ} = \sum_{m=1}^{M}w_{m}\delta_{\bz_{m}}$. Expectation with respect to measures will be denoted $\bbE_{P}[f(\mathbf{x})]$ where $\bx$ is the random variable in the expectation. For a parameter $\theta$ in a subspace $\Theta\subset\bbR^{d}$ and a function $f\colon\Theta\rightarrow\bbR$ gradients with respect to $\theta$ evaluated at a point $\theta_{0}$ are denoted by $\nabla_{\theta}f(\theta_{0})\in\bbR^{d}$. 

\textbf{Bayesian model notation:} The unknown parameter in the Bayesian setting is denoted $\theta$ and will lie in parameter space $\Theta$ with prior $\pi_{0}$ and likelihood function $l$ taking a parameter and a data sample as input, $l(\theta, \bx)$. For a training data set $\bX=\{\bx_{n}\}_{n=1}^{N}$ the full likelihood is denoted $L(\theta,\bX )  = \prod_{n=1}^{N}l(\theta,\bx_{n})$ and for a pseudo-data set $\bZ=\{\bz_{m}\}_{m=1}^{M}$ and weights $w \in\bbR^{M}$ the weighted likelihood based on pseudo-data is $L(\theta,w,\bZ) = \prod_{m=1}^{M}l(\theta,\bz_{m})^{w_{m}}$. The posterior based on $\bX$ is then $\pi_{\bX}\propto L_{\bX}\cdot\pi_{0}$ and the posterior based on $\bZ$ and $w$ is $\pi_{w,\bZ}\propto L_{w,\bZ}\cdot\pi_{0}$. 

\textbf{Non-Bayesian model notation:} A model will be denoted $F$ with final trained parameter $\theta^{*}$, meaning that $F(\theta^{*})$ denotes the trained model which itself would have inputs. The loss used for training is denoted $l(\theta,\bx)$ with $L(\theta, \bX) = \sum_{n=1}^{N}l(\theta,\bx_{n})$ denoting the accumulation of the loss over the entire training set and $L(\theta, w,\bZ) = \sum_{m=1}^{M}w_{m}l(\theta, \bz_{m})$ the weighted loss over the psuedo-data set. This notation surpresses the dependance on $F$. For example, one could have a regression model with with $\bX = \{\bx_{n}\}_{n=1}^{N}$, $\bx_{n} = (x_{n},y_{n})$ and squared-error loss $l(\theta,\bx) = (F(\theta)(x)-y)^{2}$. If a regularizer is used during the training procedure it is denoted $R(\theta)$. The combination of the loss and regularizer is denoted $\calL(\theta,\bX) = L(\theta, \bX) + R(\theta)$ and $\calL(\theta,w,\bZ) = L(\theta,w,\bZ) + R(\theta)$.

\textbf{Maximum mean discrepancy:} Maximum mean discrepancy is a kernel-based discrepancy between two measures \citep{gretton2012,muandet2017}. A user-chosen kernel is used to form the discrepancy and the choice of kernel dictates the features of the two measures that are compared. Kernels shall be denoted $k(\bx,\bx')$ and for two measures $P,Q$ the MMD is defined as 
\begin{align}
	\text{MMD}_{k}(P,Q)^{2} = \bbE_{P\times P}&[k(\bx,\bx')] + \bbE_{Q\times Q}[k(\bx,\bx')]\nonumber\\
	& - 2\bbE_{P\times Q}[k(\bx,\bx')] \label{eq:MMD_expansion},
\end{align}
which, given i.i.d.\ samples from $P,Q$, is easily estimated using empirical sums. 

It is known that every kernel can be written as $k(\bx,\bx') = \langle \varphi(\bx),\varphi(\bx')\rangle_{H}$ for some Hilbert space $H$ and some map $\varphi$ into $H$ \citep{Steinwart2008}. Such a $H$ is called a feature space and $\varphi$ is called a feature map and MMD compares $P,Q$ by using the features that $\varphi$ extracts. This is made explicit by the reformulation \citep[Lemma 4]{gretton2012} of the MMD as 
\begin{align}
	\text{MMD}_{k}(P,Q) = \lVert \bbE_{P}[\varphi(\bx)] - \bbE_{Q}[\varphi(\bx)] \rVert_{H}.\label{eq:MMD_hilbert}
\end{align}

This result shows that MMD compares the expectations of the features maps under the two measures in question, highlighting that if a feature map is more expressive then the MMD will be more discerning between the two measures. 

For example, consider when the data lies in $\bbR$ with $\varphi(\bx) = (1,\bx,\bx^{2})$ and $H = \bbR^{3}$, then the MMD will compare $\bbE_{P}[\bx^{r}]$ with $\bbE_{Q}[\bx^{r}]$ for $r = 0,1,2$, meaning the discrepancy can identify differences between measures up to the second moment. 

If a kernel is used which is able to identify arbitrary differences between measures then the kernel is called characteristic and the corresponding MMD is a valid distance \citep{Sriperumbudur2011}. This means that the feature map of the kernel has enough features to perfectly characterise any measure. For example, the Gaussian kernel essentially compares all moments of the input data and therefore is characteristic \citep{Steinwart2006}. 

\section{Threat Model and Adversary Goal}\label{sec:goals}
This section shall introduce the threat model and make a concrete definition of the goal of the adversary. Importantly, this mathematical formulation is more generel than the ``exact recovery'' scenario commonly studied, often implicitly,   in the literature and places an emphasis on how partially recovering the training set distribution still gives the adversary information, even when they don't reconstruct the exact training data set. 

\subsection{Threat Model} It is assumed that the adversary has white box access to the model, meaning the adversary knows the architecture of the model and may query all parts of it. Additionally, for the Bayesian case, samples from the posterior are given to the adversary and in the non-Bayesian case the trained parameter set is given to the adversary. These assumptions mirror the scenario for when code of a model is publicly released along with posterior samples or trained parameters. 

More specifically, in the Bayesian case the adversary can query the likelihood function $l$ and prior $\pi_{0}$ and their gradients. Whereas for the non-Bayesian case the adversary can query the model $F$, the loss function $L$ and the regularizer $R$ and their gradients. Knowledge of the model specification includes knowing the dimensionality of the training data, for example in the regression case with a data point $\bx = (x,y)$ where $x\in\bbR^{d}$ and $y\in\bbR$ it is assumed that $d$ is known since this is part of the model architecture. 

No knowledge about the distribution of the training data is assumed, in particular it is not assumed that the adversary has access to samples from the distribution which the training data came from. Critically, this assumption is different from many papers in the area of DRA \citet{Kaissis2023_hypothesis,Balle2022_informed,Hayes203_dpsgd} and means that the attacker entirely relies upon the recovery algorithm to produce insights into the unknown training data. 

\subsection{Adversary Goal} 
The following is the mathematical description of the data reconstruction problem that will be used throughout this paper. 
\begin{definition}\label{def:goal}
	If $\bX = \{\bx_{n}\}_{n=1}^{N}$ is the training data set then the aim of the adversary is to approximate the empirical distribution $P_{\bX} = \sum_{n=1}^{N}\delta_{\bx_{n}}$ with $P_{w,Z} = \sum_{m=1}^{M}w_{m}\delta_{\bz_{m}}$ where $w = \{w_{m}\}_{m=1}^{M}$ are scalar weights, $\bZ = \{\bz_{m}\}_{m=1}^{M}$ are pseudo-data points and $M$ is an adversary chosen constant.
\end{definition} 

If this goal is achieved perfectly, meaning $P_{\bX} = P_{w,\bZ}$, then $N=M$, $w_{m} = 1\:\forall m$ and $\bZ = \bX$ so the adversary has perfectly reconstructed the training data. However, even if perfect reconstruction isn't achieved but $P_{\bX}$ and $P_{w,\bZ}$ are still somehow close then the adversary will still have gained some information about the training data distribution without violating the mathematical framework. A natural question is in what metric or divergence is the approximation taken? This is a choice by the attacker and the choice used in our paper is explained in Definition \ref{def:Bayes_reco}.

Weights are used by the adversary to make the reconstruction task easier and to match the scales. If weights were not used then there would always be a large difference between the target and approximate empirical measure whenever $M$ and $N$ are different, which is likely to happen as $N$ is unknown to the adversary. 

The following example emphasises how the empirical measure approach allows for broader notions of reconstruction that are helpful to an adversary.

\begin{example}
	Suppose a classifier is trained on cats and dogs and the adversary has ran a reconstruction algorithm which produces images of cats and dogs but not the \textit{exact} images used in the training data. This outcome is still an approximation to the empirical distribution and the adversary is still learning the sort of data to expect from the training data, without performing perfect reconstruction.
\end{example}

\section{Recovering Training Data from Bayesian Posteriors}\label{sec:Bayesian}
This section shall outline two primary contributions of the paper. First, a score matching method to reconstruct features of training data from Bayesian posteriors, which will later be shown to be a natural generalisation of an existing method to reconstruct data from non-Bayesian models. Second, a result to characterise which training data features can be reconstructed from a given Bayesian model, giving a characterisation of the potential performance of the reconstruction method. 

The former is the first of its kind in the literature, since the current literature focuses purely on non-Bayesian models. The latter aims to give a concrete characterisation of how model complexity impacts reconstruction algorithms as these two factors have been seen to be intimately related in numerical experiments but have not been analysed theoretically. 

\subsection{Recovering Training Data}\label{subsec:bayesian_recovery}
The idea of how to recover training data from a Bayesian posterior is a simple trick, achieved by reversing the logic for fitting statistical models. 

Typically, when fitting a generative statistical model a training data set is observed, for example some pictures of cats and dogs, and then parameters of the model are fit so that the model then produces samples similar to the observed data. A way of doing this without having to sample from the model at each step, and without needing normalising constants, is to use score-based methods such as Fisher divergence or its modifications \citep{hyvarinen2005,song2019sliced}. Indeed, using Fisher divergence, or other score-based divergences has become the defacto method when fitting generative models \citep{Song2020_how_to_train}. 

How does this standard set up relate to the reconstruction problem? In the reconstruction problem, the adversary has observed parameter samples from the posterior and wants to find the training data which would produce such samples. This is exactly the same as the above but with training data and parameters being re-labelled and the user uses the posterior based on psuedo-data and weights as their approximating score function, rather than a neural network of some kind. 

The discussion in the rest of our paper will focus on the standard Fisher divergence with analogous results for the sliced version, which has better numerical properties, presented in the Appendix. Analogous results for even more sophisticated score matching methods such as de-noising are left as future work. 

\begin{definition}
	The Fisher divergence between the posterior $\pi_{\bX}$ and the weighted poster $\pi_{w,\bZ}$ using weights and pseudo-data is 
	\begin{align}
		&\emph{FD}(\pi_{\bX},\pi_{w,\bZ}) \nonumber\\
		& = \frac{1}{2}\int \lVert \nabla_{\theta}\log\pi_{\bX}(\theta) - \nabla_{\theta}\log\pi_{w,\bZ}(\theta)\rVert_{\Theta}^{2}\mathrm{d}\pi_{\bX}(\theta).\label{eq:FD}
	\end{align}
\end{definition}

The advantage of the Fisher divergence is that the dependence on $\nabla\log\pi_{\bX}$, which is unavailable during an attack as its computation depends on the unknown data $\bX$, can be removed by an integration-by-parts trick, resulting in 
\begin{align}\label{eq:FD_IBP}
	\text{FD}(\pi_{\bX},\pi_{w,\bZ}) & = \bbE_{\pi_{\bX}}[\text{Tr}(\nabla_{\theta}^{2}\log\pi_{w,\bZ}(\theta))] \nonumber\\
	& + \frac{1}{2}\bbE_{\pi_{\bX}}[\lVert \nabla_{\theta}\log\pi_{w,\bZ}(\theta)\rVert^{2}] \\
	& + C,\nonumber
\end{align}
where $C$ is a constant that does not depend on $w,\bZ$ and so can be ignored when it comes to minimizing the FD. This relies on only very mild assumptions on the regularity of the score functions \citep{hyvarinen2005}.

As part of the threat model the following assumption is made.

\begin{assumption}\label{ass:bayes}
	The adversary has access to $T$ samples from the posterior $\pi_{\bX}$.
\end{assumption}

This makes the Fisher divergence, and its gradient, straight forward to approximate using the samples from $\pi_{\bX}$ and the gradients $\nabla\log\pi_{w,\bZ}$, which are assumed to be available within the threat model. This results in the estimator
\begin{align*}
	\text{FD}(\pi_{\bX},\pi_{w,\bZ})& \approx \frac{1}{T}\sum_{t=1}^{T}\text{Tr}(\nabla_{\theta}^{2}\log\pi_{w,\bZ}(\theta_{t})) \\
	& + \frac{1}{2T}\sum_{t=1}^{T}\lVert \nabla_{\theta}\log\pi_{w,\bZ}(\theta_{t})\rVert^{2} \\
	& + C.
\end{align*}

\begin{definition}\label{def:Bayes_reco}
	The training data reconstruction problem based on Fisher divergence is defined as
	\begin{align*}
		w,\bZ = \argmin_{w,\bZ}\emph{FD}(\pi_{\bX},\pi_{w,\bZ}).
	\end{align*}
\end{definition}

This invite a few comments. First, it shows that the parameters that are being optimised with respect to are exactly those appearing in Definition \ref{def:goal}. If the goal was perfectly achieved then the Fisher divergence would be zero. This objective gives an approach to recover training data that results in the same model as the one trained on the true, unknown training data. Though standard Fisher divergence has been used in this definition, sliced Fisher divergence could also be used as it has the same properties of being easily estimated given posterior samples \citep{song2019sliced}. The next subsection will describe how well this method is expected to work, given the features of the data the model uses. Finally, an adversary may wish to regularize the weights and psudeo-data to use any prior knowledge they may have, for example a total variation norm if they are trying to reconstruct images, this is common in the literature \citet{buzaglo2023deconstructing}.

\subsection{Characterisation of Reconstruction}\label{subsec:bayesian_char}
Before the main result of this subsection a motivating example is given. The aim of this example is to catalyse the question of what features of training data should one expect, or in fact even hope, to be vulnerable to reconstruction from a posterior by looking at an extremely simple model. A similar discussion was provided in \citet{Manousakas2020} in the context of psuedo-coresets.

\begin{example}[Gaussian mean location]\label{exp:mean_loc}
	Consider the Gaussian mean location model, the aim of which is to infer the mean of observed data $\{\mathbf{x}_{n}\}_{n=1}^{N}\subset\bbR^{d}$. Under a standard multivariate Gaussian prior $\pi_{0} = N(0,I)$ and standard Gaussian likelihood $l(\theta, \bx) = \exp(-\frac{1}{2}\lVert \theta - \bx\rVert^{2})$ the posterior is a multivariate Gaussian with mean 
	$\mu = \frac{1}{N+1}\sum_{n=1}^{N}\bx_{n}$ and covariance matrix $\Sigma = (N+1)^{-1}I$. 
	
	Now suppose that for some $M\in\bbN$ the data set $\bZ = \{\bz_{m}\}_{m=1}^{M}$ was observed and the likelihood was weighted using $w = (w_{m})_{m=1}^{M}\in\bbR^{m}$ but the same prior was used. Let $S_{w} = \sum_{m=1}^{M}w_{m}$ then the posterior would still be Gaussian with mean $\frac{1}{S_{w} + 1}\sum_{m=1}^{M}w_{m}\bz_{m}$ and covariance matrix $(S_{w} + 1)^{-1}I$.
	
	Looking at this, ones sees that regardless of the value of $M$ if you started with any data set such that $S_{w} = N$ and $\sum_{m=1}^{M}w_{m}\bz_{m} = \sum_{n=1}^{N}\bx_{n}$ then the same posterior would be recovered. This shows that only the total number of data samples $N$ and the sum of the data samples $\sum_{n=1}^{N}\bx_{n}$ is needed to recover the same posterior. 
	
	Therefore, if a reconstruction attack was performed, the intuition is that only the total number of points and sum of the training data could be recovered as, given the model specification, that is the minimal information needed to characterise the posterior. This is equivalent to the number of points and the sum satisfying Bayesian sufficiency for the posterior.
\end{example}

This example shows that for a simple model, only a simple statistic is needed to fully characterise the posterior. Therefore, if one is trying to optimise data to attain the same posterior one would struggle to reveal more information about the data beyond this simple statistic. 

Example \ref{exp:mean_loc} inspires the \textit{ansatz} that a characterisation of the features of training data that can be recovered from posterior samples should somehow depend on the complexity of the model. \textbf{The more complex a model, the more complex the statistic to characterise it, hence the more complex the data which can be recovered from the model}. The following theorem shows this is indeed the case by equating the Fisher divergence between the posterior based on training data and the posterior based on weighted pseudo-data with a MMD using a kernel whose features depend on the model. 

\begin{theorem}\label{thm:characterise}
	Let $\bX = \{\bx_{n}\}_{n=1}^{N}, \bZ = \{\bz_{m}\}_{m=1}^{M}$ be two data sets and $w = \{w_{m}\}_{m=1}^{M}$ a set of scalar weights. Let $\pi_{0}$ be a prior for an unknown parameter lying in $\Theta \subset \bbR^{d}$ and $l$ be a likelihood function. For  ${\pi_{\bX}(\theta)\propto\prod_{n=1}^{N}l(\theta,\bx_{n})\cdot\pi_{0}(\theta)}$ and ${\pi_{w,\bZ}(\theta)\propto\prod_{m=1}^{M}l(\theta,\bz_{m})^{w_{m}}\cdot\pi_{0}(\theta)}$  
	\begin{align*}
		\emph{FD}(\pi_{\bX},\pi_{w,\bZ}) = \frac{1}{2}\:\emph{MMD}_{k}(P_{\bX}, P_{w,\bZ})^{2} 
	\end{align*}
	where 
	\begin{align*}
		k(\bx,\bx') = \int_{\Theta}\langle\nabla_{\theta}\log l(\theta,\bx),\nabla_{\theta}\log l(\theta,\bx')\rangle_{\bbR^{d}} \mathrm{d}\pi_{\bX}(\theta)
	\end{align*}
	and $P_{\bX} = \sum_{n=1}^{N}\delta_{\bx_{n}}, P_{w,\bZ} = \sum_{m=1}^{M}w_{m}\delta_{\bz_{m}}$ are the un-normalised empirical data measures. 
\end{theorem}

\begin{proof}
	First note 
	\begin{align}
		& \nabla_{\theta}\log\pi_{\bX}(\theta) - \nabla_{\theta}\log\pi_{w,\bZ}(\theta)\nonumber \\
		& = \sum_{n=1}^{N}\nabla_{\theta}\log l(\theta,\bx_{n}) - \sum_{m=1}^{M}w_{m}\nabla_{\theta}\log l(\theta,\bz_{m})\nonumber\\
		& = \bbE_{P_{\bX}}[\nabla_{\theta}\log l(\theta,\bx)] - \bbE_{P_{w,\bZ}}[\nabla_{\theta}\log l(\theta,\bz)].\label{eq:expand}
	\end{align}
	Next, substitute \eqref{eq:expand} into \eqref{eq:FD} and expand the squared norm as an inner product, not including the outer expectation with respect to $\pi_{\bX}$ for now,
	\begin{align}
		&\lVert\nabla_{\theta}\log\pi_{\bX}(\theta)\rVert_{\bbR^{d}}^{2} - 2\langle \nabla_{\theta}\log\pi_{\bX}(\theta),\nabla_{\theta}\log \pi_{w,\bZ}(\theta)\rangle_{\bbR^{d}}\nonumber \\
		& + \lVert\nabla_{\theta}\log\pi_{w,\bZ}(\theta)\rVert_{\bbR^{d}}^{2}\nonumber\\
		& = \bbE_{P_{\bX}\times P_{\bX}}[\langle \nabla_{\theta}\log l(\theta,\bx),\nabla_{\theta}\log l(\theta,\bx')\rangle_{\bbR^{d}}]\nonumber \\
		& - 2\bbE_{P_{\bX}\times P_{w,\bZ}}[\langle \nabla_{\theta}\log l(\theta,\bx),\nabla_{\theta}\log l(\theta,\bx')\rangle_{\bbR^{d}}]\label{eq:expansion_2}\\
		& + \bbE_{P_{w,\bZ}\times P_{w,\bZ}}[\langle \nabla_{\theta}\log l(\theta,\bx),\nabla_{\theta}\log l(\theta,\bx')\rangle_{\bbR^{d}}].\nonumber
	\end{align} 
	Finally, adding the expectation with respect to $\pi_{\bX}$ that is in the FD and noting that 
	\begin{align*}
		\int\langle \nabla_{\theta}\log l(\theta,\bx),\nabla_{\theta}\log l(\theta,\bx')\rangle_{\bbR^{d}}\mathrm{d}\pi_{\bX}(\theta) = k(\bx,\bx'),
	\end{align*}
	shows that each term in \eqref{eq:expansion_2} corresponds directly to a term in \eqref{eq:MMD_expansion} which completes the proof. 
\end{proof}

Multiple remarks are in order. A connection between data reconstruction and MMD was previously used in \citet{Loo2023} as a proof tool, rather than as an equivalence to analyse the potency of reconstruction attacks. The connection to Fisher divergence was not highlighted and nor were Bayesian models studied. A result connecting MMD and discrepancy between posteriors was derived by \citet{Wynne2023} using Bayes Hilbert spaces but this did not involve Fisher divegence. Theorem \ref{thm:characterise} shows that if a training data reconstruction attack aims to reduce the Fisher divergence between the posterior and pseudo-posterior in an attempt to recover the training data, then this is equivalent to minimising the MMD between the empirical training data measure and the empirical weighted pseudo-data measure. The kernel of this MMD has $L^{2}(\Theta, \pi_{\bX})$ as its hilbert space and $\varphi(\bx) = \nabla_{\theta}\log l(\cdot,\bx)$ as its feature map. \textbf{Therfore, the gradient of the log-likelihood function completely determines the features that can be reconstructed and the posterior determines the weight placed on these features}. 

The Gaussian mean location example can be continued to show how the \textit{ansatz}, which was divined from looking at the explicit formulas of the posterior, matches with the feature map in the above result. 

\begin{example}[Gaussian mean location continued]
	Under the Gaussian mean location model 
	\begin{align*}
		\varphi(\bx)(\theta) = \nabla_{\theta}\log l(\theta,\bx) = -(\theta - \bx),
	\end{align*}
	meaning that
	\begin{align*}
		\bbE_{P_{\bX}}[\varphi(\bx)(\theta)] & = -N\theta + \sum_{n=1}^{N}\bx_{n}\\
		\bbE_{P_{w,\bZ}}[\varphi(\bz)(\theta)] & = -\left(\sum_{m=1}^{M}w_{m}\right)\theta + \sum_{m=1}^{M}w_{m}\bz_{m}.
	\end{align*}
	Since MMD is equal to the difference between these two expressions in $L^{2}(\Theta,\pi_{\bX})$, see \eqref{eq:MMD_hilbert}, having zero FD, hence zero MMD, between the posterior and weighted posterior means the two above expressions are equal as functions of $\theta$ wherever $\pi_{\bX}$ has a non-zero value. In this example $\pi_{\bX}$ is Gaussian so it is non-zero everywhere. This implies that $N = \sum_{m=1}^{M}w_{m}$ and $\sum_{n=1}^{N}\bx_{n} = \sum_{m=1}^{M}w_{m}\bz_{m}$. This shows that after optimizing $w,\bZ$ with respect to FD (hence with respect to MMD by virtue of Theorem \ref{thm:characterise}) the information of the training set that can be recovered from $w,\bZ$ is the total number of points and the sum of the training points. This matches the intuition gained in the previous example by looking at the explicit expressions for the posterior.  
\end{example}

Theorem \ref{thm:characterise} shows that the more expressive the feature map, the more of the features of the training data can be recovered. The model and its derivative implicitly plays a role in the feature map as it is part of the likelihood function, which means that the more features the model extracts of the data the more expressive the feature map. For example, if the model is a neural network then increasing the depth and width increases the features extracted from the training data and therefore increases the features which can be reconstructed from the training data. This was observed numerically in the non-Bayesian case by  \citet{Haim2022,Loo2023}. \textbf{This poses an important conflict in model privacy as more features used by the model typically means better model performance but Theorem \ref{thm:characterise} shows that more features used by the model means more features of the training data can be recovered}. 

Theorem \ref{thm:characterise} highlights the impact that training data set size has on the training data reconstruction problem. Because the target measures are un-normalised, their norm, which is a measure of how ``complex'' they are, grows with the number of data points.  This can be seen explicitly as 
\begin{align*}\lVert P_{\bX}\rVert_{H}^{2} & = \sum_{n=1}^{N}k(\bx_{n},\bx_{n}) \\
	& = \sum_{n=1}^{N}\int \lVert \nabla_{\theta}\log l(\theta,\bx_{n})\rVert_{\Theta}^{2}\mathrm{d}\pi_{\bX}(\theta)\\
	& \xrightarrow{N\rightarrow\infty}\infty.
\end{align*}
This can intuitively be seen as an issue by considering stadard function approximation. If you had a fuction $f$ on $[0,1]$ that you were trying to approximate using a set method, it would be easier to approximate it if $\lVert f\rVert_{L^{2}([0,1])} = 1$ rather than $\lVert f\rVert_{L^{2}([0,1])} = 1000$ because the latter is more ``complex''.

This shows that unless the attacker is able to adapt the complexity of the approximating measure $P_{w,\bZ}$ then the un-normalised nature of the target measure $P_{\bX}$ will make it harder to recover the training data if the complexity of the initialisation of $P_{w,\bZ}$, dictated by $\sum_{m=1}^{M}w_{m}$, is far from the complexity of $P_{\bX}$. This provides an explanation for the numerics that were observed in the non-Bayesian case by \citet{Haim2022,Loo2023}. Using a model with more features was seen to combat the issue of larger training data sets causing worse training data reconstruction but an exact characterisation of the trade off is an open problem. 

This section shall conclude with a final worked example, further highlighting how the more features a model extracts from data the greater the features of the training set can be recovered.

\begin{example}[Bayesian linear regression]\label{exp:BLR}
	Consider Bayesian linear regression with training data $\bX = \{\bx_{n}\}_{n=1}^{N}$ where $\bx_{n} = (x_{n},y_{n})$ with $x_{n}\in\bbR^{d}, y_{n}\in\bbR$. A feature vector $\psi(\bx)\in\bbR^{d'}$ will be used as features. For example $d=1,d'=3$ and $\psi(x) = (1,x,x^{2})$. The pseudo-data will be denoted $\bZ = \{\bz_{m}\}_{m=1}^{M}$ with $z_{m} = (z_{m},u_{m}), z_{m}\in\bbR^{d},u_{m}\in\bbR$. A standard Gaussian multivariate prior is used for the unknown coefficients $\theta\in\Theta = \bbR^{d'}$ and the likelihood is Gaussian $l(\theta,\bx) \propto\exp(-\frac{1}{2}( \langle\theta,\psi(x)\rangle_{\bbR^{d'}} - y)^{2})$. 
	
	The feature map in the kernel for MMD is 
	\begin{align*}
		\varphi(\bx)(\theta) & = \nabla_{\theta}\log l(\theta,\bx)  \\
		& = -\psi(x)\psi(x)^{\top}\theta + \psi(x)y,
	\end{align*}
where $\psi(x)\psi(x)^{\top}\in\bbR^{d'\times d'}$ is the outer product of the features of the input data. Therefore, the two un-normalised expectations of the features with respect to the data distrbutions are
\begin{align*}
	\bbE_{P_{\bX}}[\varphi(\bx)(\theta)]  = & -\left(\sum_{n=1}^{N}\psi(x_{n})\psi(x_{n})^{\top}\right)\theta \\
	& + \sum_{n=1}^{N}\psi(x_{n})y_{n} \\
	\bbE_{P_{w,\bZ}}[\varphi(\bz)(\theta)] & = -\left(\sum_{m=1}^{M}w_{m}\psi(z_{m})\psi(z_{m})^{\top}\right)\theta \\
	& + \sum_{m=1}^{M}w_{m}\psi(z_{m})u_{m}.
\end{align*} 
If the FD is minimised to zero then these two expressions must be equal as functions of $\theta$. Setting $\theta = 0$, which is in the support of the posterior, gives $\sum_{n=1}^{N}\psi(x_{n})y_{n} = \sum_{m=1}^{M}w_{m}\psi(z_{m})y_{m}$ and hence also the expectations of the outer product of the features $\psi(x)\psi(x)^{\top}$ with respect to $P_{\bX}$ and $P_{w,\bZ}$ must be equal. 

Continuing the example of $\psi(x) = (1,x,x^{2})$ then the outer product of features captures $x^{r}$ for $r = \{0,1,2,3,4\}$. Therefore, if the FD is minimised to zero then the first five moments of the training data can be recovered. If the feature vector included higher polynomial moments then higher moments of the training data could be recovered. 
\end{example}

\section{Recovering Training Data from Non-Bayesian Models}\label{sec:non_bayesian}
This section will outline how the findings of the previous section can be applied to non-Bayesian models. It will be shown how the derivations coincide and generalise current methods in the literature. The adversary goal is still outlined in Definition \ref{def:goal} as being the reconstruction of the un-normalised empirical distribution of the training data set. 

\subsection{Recovering Training Data}\label{subsec:non_bayesian_recovery}
The move from Bayesian to non-Bayesian models is made with two ingredients. First, by viewing the final parameter obtained at the end of training, denoted $\theta^{*}\in\bbR^{d}$, as a posterior distribution that only consists of this parameter i.e. a Dirac measure on $\theta^{*}$ denoted $\delta_{\theta^{*}}$. Second, by using the relationship between the likelihood and prior pair with loss and regularizer. 

Starting with the Fisher divergence \eqref{eq:FD}, if one replaces $\mathrm{d}\pi_{\bX}$ with $\mathrm{d}\delta_{\theta^{*}}$ for a single parameter $\theta^{*}$ and replaces $\nabla_{\theta}\log\pi_{\bX}(\theta)$, $\nabla_{\theta}\log\pi_{w,\bZ}(\theta)$ with $\nabla_{\theta}\calL(\theta,\bX)$, $\nabla_{\theta}\calL(\theta,w,\bZ)$, respectively, then \eqref{eq:FD} becomes
\begin{align}
	&\lVert \nabla_{\theta}\calL(\theta^{*},\bX) - \nabla_{\theta}\calL(\theta^{*},w,\bZ)\rVert_{\bbR^{d}}.\label{eq:non_Bayes_FD}
\end{align}

This swap of $\log\pi_{\bX}(\theta)$ for $\calL(\theta,\bX)$ comes from how a log-likelihood can correspond to a loss function and a log-prior to a regularizer.

At this point one might want to try and minimize \eqref{eq:non_Bayes_FD} with respect to $w,\bZ$ to reconstruct $P_{\bX}$ but the term $\nabla_{\theta}\calL(\theta^{*},\bX)$ is intractable as it depends on the unknown data $\bX$. In the Bayesian case an integration-by-parts trick is used to remove the $\nabla_{\theta}\log\pi_{\bX}(\theta)$ term to make the divergence numerically tractable. In the current scenario the following assumption is made in the literature \citep{Haim2022,buzaglo2023deconstructing}.

\begin{assumption}\label{ass:param}
	The parameters that are released to the adversary $\theta^{*}\in\bbR^{d}$ satisfy $\nabla_{\theta}\calL(\theta^{*},\bX) = \mathbf{0}\in\bbR^{d}$.
\end{assumption}

In plain language, this assumption states the the model has been trained to a local minimum of the objective $\calL$. This highlights a natural trade off between Assumption \ref{ass:bayes} and Assumption \ref{ass:param}. In the former it is assumed that the adversary has samples from $\pi_{\bX}$, which naturally satisfy $\int\nabla_{\theta}\log\pi_{\bX}(\theta)\mathrm{d}\pi_{\bX}(\theta) = \mathbf{0}$, which is the Bayesian analogy to $\nabla_{\theta}\calL(\theta^{*},\bX) = \mathbf{0}$.

This assumptions leads to the following reconstruction problem.
\begin{definition}\label{def:non_Bayes_goal}
	Under Assumption \ref{ass:param} the reconstruction problem in the non-Bayesian case is 
	\begin{align}\label{eq:non_bayesian_obj}
		w,\bZ = \argmin_{w,\bZ}\lVert \nabla_{\theta}\calL(\theta^{*},w,\bZ)\rVert_{\bbR^{d}}.
		\end{align}
\end{definition}  

This is very similar to the minimisation targets which have been derived for training data reconstruction methods in the literature \citep{Haim2022,Loo2023, buzaglo2023deconstructing}. The main difference is that in the existing literature the weights are not viewed as an object of interest and are instead thrown away after optimisation, with focus purely on the pseudo-data. The entire un-normalised measure is not viewed as the overall object for reproduction in existing methods, instead the focus is on the psudeo-data reconstructing the true data perfectly e.g. perfect image reconstruction. In contrast, the weights play a critical role in Definition \ref{def:goal}, the primary goal of the adversary in the present context.

The objective \eqref{eq:non_bayesian_obj} is arrived at in the literature via more complex mathematics focusing around particular scenarios, for example KKT conditions in \citet{Haim2022}. The more simple logic in this section shows that it can instead be viewed as a natural consequence of starting at Fisher divergence and substituting in a Dirac mass centered on the trained parameters for the posterior. 

Equipped with the objective in Definition \ref{def:non_Bayes_goal}, the attacker then uses which ever minimisation method they prefer to obtain $w,\bZ$ in an attempt to fulfill the goal in Definition \ref{def:goal}. As was the case for the objective in the Bayesian case, an adversary may also want to regularize the data somehow given any prior knowledge they have.

\subsection{Characterisation of Reconstruction}\label{subsec:non_bayesian_char}
Analogous to how a characterisation in terms of MMD can be given to the objective in the Bayesian case, an equivalance can be drawn between \eqref{eq:non_Bayes_FD} and an MMD with a particular kernel.

\begin{theorem}\label{thm:non_Bayes_characterise}
	Let $\bX = \{\bx_{n}\}_{n=1}^{N},\bZ = \{\bz_{m}\}_{m=1}^{M}$ be two data sets and $w = \{w_{m}\}_{m=1}^{M}$ a set of scalar weights. Then, for $\theta^{*}\in\bbR^{d}$,
	\begin{align*}
		\emph{MMD}_{k}(P_{\bX},P_{w,\bZ})  & = \lVert\nabla_{\theta}\calL(\theta^{*},\bX) - \nabla_{\theta}\calL(\theta^{*},w,\bZ)\rVert_{\bbR^{d}}\\
		& = \lVert\nabla_{\theta}L(\theta^{*},\bX) - \nabla_{\theta}L(\theta^{*},w,\bZ)\rVert_{\bbR^{d}},
	\end{align*}
where 
\begin{align*}
	k(\bx,\bx') = \langle \nabla_{\theta}l(\theta^{*},\bx), \nabla_{\theta}l(\theta^{*},\bx')\rangle_{\bbR^{d}}.
\end{align*}
\end{theorem}

\begin{proof}
	The proof is simple rearranging of sums. First, the second equality is immediate since $\calL$ can be replaced by $L$ because the $R$ terms in $\calL$ cancel out as they do not depend on the choice of $\bX,w,\bZ$. 
	
	Then, note $\nabla_{\theta}L(\theta^{*},\bX) = \bbE_{P_{\bX}}[\nabla_{\theta}l(\theta^{*},\bx)]$ and $\nabla_{\theta}L(\theta^{*},w,\bZ) = \bbE_{P_{w,\bZ}}[\nabla_{\theta}l(\theta^{*},\bx)]$ which proves the result using the identity \eqref{eq:MMD_hilbert} with $H=\bbR^{d}$ and $\varphi(\bx) = \nabla_{\theta}l(\theta^{*},\bx)$. 
\end{proof}
A similar result was used in a proof by \citet{Loo2023} but not with the additional perspective of using un-normalized measures. Theorem \ref{thm:non_Bayes_characterise} is the non-Bayesian analogy to Theorem \ref{thm:characterise}. The consequence is that the same remarks from the Bayesian case can be drawn for the non-Bayesian case. In particular, the more features in the model - for example more depth or width in a neural network - lead to more features being recovered in the reconstruction attack since more features are involved in the kernel in Theorem \ref{thm:non_Bayes_characterise}. This is because more features in a model means more expressive $\nabla_{\theta} l(\theta^{*},\cdot)$ functions and hence more discerning kernels. Additionally, the more points in the training data the more complex the task of reconstructing the training data and therefore the worse the reconstruction performance. 

Theorem \ref{thm:non_Bayes_characterise} explains the findings in the numerical experiments performed by \citet{Haim2022,buzaglo2023deconstructing}. In particular, \citet[Figure 7]{buzaglo2023deconstructing} shows the results of a reconstruction attack and how it depends on the size of the training set and number of neurons per layer of a network. The attack coincides with the one outlined in this section. It is shown that increasing the number of training points reduced reconstruction quality, explained by the discussion in Section \ref{sec:Bayesian}, and increasing the number of neurons per layer increases reconstruction quality, explained by Theorem \ref{thm:non_Bayes_characterise} as more neurons per layer means more features to be matched in the kernel feature map.

\section{Numerics}\label{sec:numerics}
This section shall present a simple example of employing the reconstruction method to recover data in a Bayesian linear regression example, using a model and posterior samples from \texttt{posteriodb} \citep{posteriordb} an open database of Bayesian models and posterior samples. Code is available at \url{https://github.com/ggcode-spec/score_data_reconstruction}. The sliced version of the Fisher divergence will be used due to its better numerical properties than standard Fisher divergence \citep{song2019sliced}. See the Appendix for results regarding sliced Fisher divergence analogous to those in Section \ref{sec:Bayesian}.

The intention of this section is to show that the framework presented in previous sections for data reconstructions goes beyond the ``perfect reconstruction'' implicit aim in the literature. Instead, by using the chracteration in Theorem \ref{thm:characterise} it will be shown that statistics of the training data can be extracted with the generic optimisation problem in Definition \ref{def:Bayes_reco} and that these statistics are captured by the weighted empirical measure, rather than the raw pseudo-data itself. 

The model is the \texttt{kidscore\_momiq} model in \texttt{posteriordb} \citep{posteriordb}. This model is featured in \citet[Chapter 3]{Gelman_Hill_2006} and involves predicting cognitive test scores of three and four year old children by using their mothers IQ test. This model is being used because the data and gold standard posterior samples of the model are easily available from \texttt{posteriordb} and the model is simple enough to have interpretable data reconstruction results, with the data being real life rather than synthetic. 

The model is Bayesian linear regression, with two unknown parameters $\theta = (\beta,\sigma)$ with $\beta\in\bbR^{2}, \sigma\in\bbR$. A single data point is denoted $\bx_{n} = (x_{n},y_{n})$ where $x_{n} = (1,s_{n})$ with $s_{n}\in\bbR$ being the mother IQ test score and the $1$ as an intercept and $y_{n}\in\bbR$ is the child score. The total number of training data samples is $N=434$. The reconstructed data will be parameterised with a user choosen value of $M$, weights $w\in\bbR^{M}$ and psuedo data $\bZ = \{\bz_{m}\}_{m=1}^{M}$ where $\bz_{m} = (1,r_{m})$ with $r_{m}\in\bbR$ representing the mother IQ test score and $u_{m}$ will represent the child test score. 

The likelihood is Gaussian 
\begin{align*}
	l(\theta,\bx) = \frac{1}{\sqrt{2\pi\sigma^{2}}}\exp\left(-\frac{1}{2\sigma^{2}}(\langle \beta, x\rangle - y)^{2}\right)
\end{align*}
and a flat prior is placed on $\beta$ and a Cauchy prior is placed on $\sigma$ with scale $2.5$. 

Even though the sliced version of Fisher divergence (SFD) \citep{song2019sliced} is being used, by Lemma \ref{lem:SFD_FD} we know that this is equivalent to minimising the Fisher divergence. Therefore, using the same logic as Example \ref{exp:BLR} we can work out what properties of the data we should expect to be able to recover. 

\begin{lemma}\label{lem:reco_vals}
	If the \emph{SFD} between $\pi_{\bX}$ and $\pi_{w,\bZ}$ is zero then 
	\begin{align*}
		\bX^{\top}\bX & = \sum_{m=1}^{M}w_{m}\bz_{m}\bz_{m}^{\top} \\
		\bX^{\top}\by & = \sum_{m=1}^{M}w_{m}\bz_{m}^{\top}u_{m} \\
		\sum_{n=1}^{N}y_{n}^{2} & = \sum_{m=1}^{M}w_{m}u_{m}^{2}
	\end{align*}
\end{lemma}

The proof is in the Appendix. All the terms on the left hand side of these equations are the sufficient statistics for Bayesian linear regression, this shows that Theorem \ref{thm:characterise} is an alternative theoretical tool to recover such statistics. 

As the first entry of each data point is $1$ the gram matrix $\bX^{\top}\bX$ is $2\times 2$ with entries $(N,\sum_{n=1}^{N}s_{n},\sum_{n=1}^{N}s_{n},\sum_{n=1}^{N}s_{n}^{2})$, with the middle entry repeated. Recall that $s_{n}$ is the value in data point $\bx_{n}$ that is the mothers score. The entries of $\sum_{m=1}^{M}w_{m}\bz_{m}\bz_{m}^{\top}$ are $(\sum_{m=1}^{M}w_{m}, \sum_{m=1}^{M}w_{m}r_{m},\sum_{m=1}^{M}w_{m}r_{m},\sum_{m=1}^{M}w_{m}r_{m}^{2})$, where $r_{m}$ is the reconstructed mothers score. This means we would be able to reconstruct the total number of points, the empirical mean of the mothers scores and the empirical variance of the mothers scores from the gram matrix.

As the first entry of every row of $\bX$ is $1$, the expression $\bX^{\top}\by$ has $\sum_{n=1}^{N}y_{n}$ as its first entry, so we also get the total sum of the childrens scores, and combined with $\sum_{n=1}^{N}y_{n}^{2} = \sum_{m=1}^{M}w_{m}u_{m}^{2}$ and the total number of data points $N = \sum_{m=1}^{M}w_{m}$ this lets us reconstruct the empirical mean and empirical variance of the childrens scores. 

To validate these deductions, sliced score matching is performed on the objective in Definition \ref{def:Bayes_reco} with respect to $w$, $\bZ$. Varying choices of the number of pseudo data points $M=\{50, 100, 200, 400, 800, 1600\}$ are used. The last $T=1000$ samples from the reference posterior from \texttt{posteriordb} are used. The fact that $1$ is the first component of each data point is viewed as part of the model known to the adversary. The mothers IQ data is initialized as standard normal, and the child scores are initialised as $u_{m} = \langle \bar{\beta},\bz_{m}\rangle + \bar{\sigma}\varepsilon_{m}$ where $\bar{\beta},\bar{\sigma}$ are the means of the $\beta$, $\sigma$ samples the adversary has access to and $\varepsilon_{m}$ is i.i.d. standard normal. The weights $\{w_{m}\}_{m=1}^{M}$ are all initialised as one. For the slicing distribution, $L=10$ samples of standard multivariate normal are used at each iteration. Optimistion is done using Adam in Optax with learning rates for $r_{m},u_{m},w_{m}$ all set to $0.001$.\

\begin{figure}
	\includegraphics[width=\linewidth]{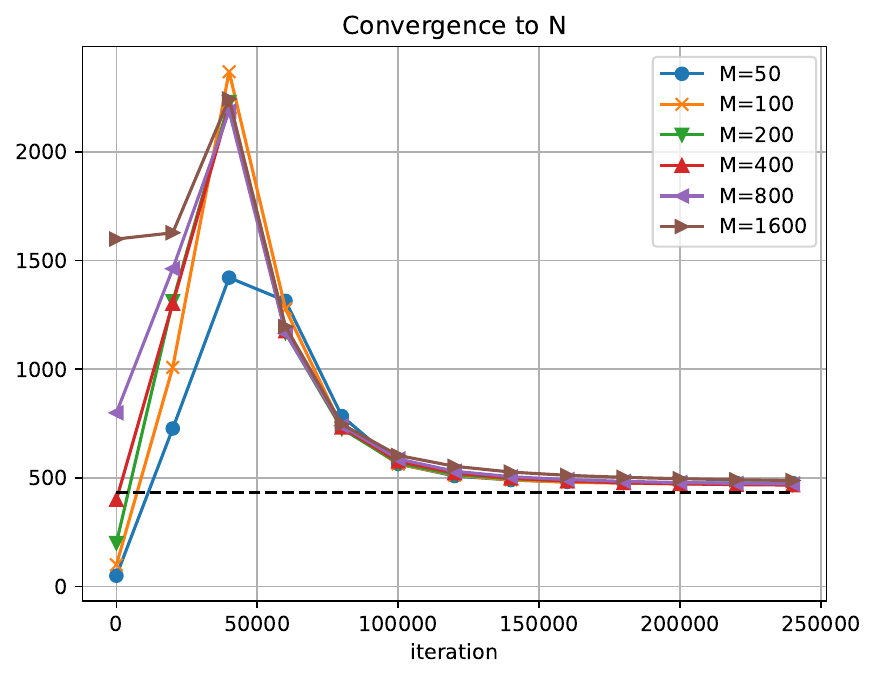}\caption{Convergence of $\sum_{m=1}^{M}w_{m}$ to $N$ where $N$ is the number of training data points.}
	\label{fig:N}
\end{figure}

\begin{figure}
	\includegraphics[width=\linewidth]{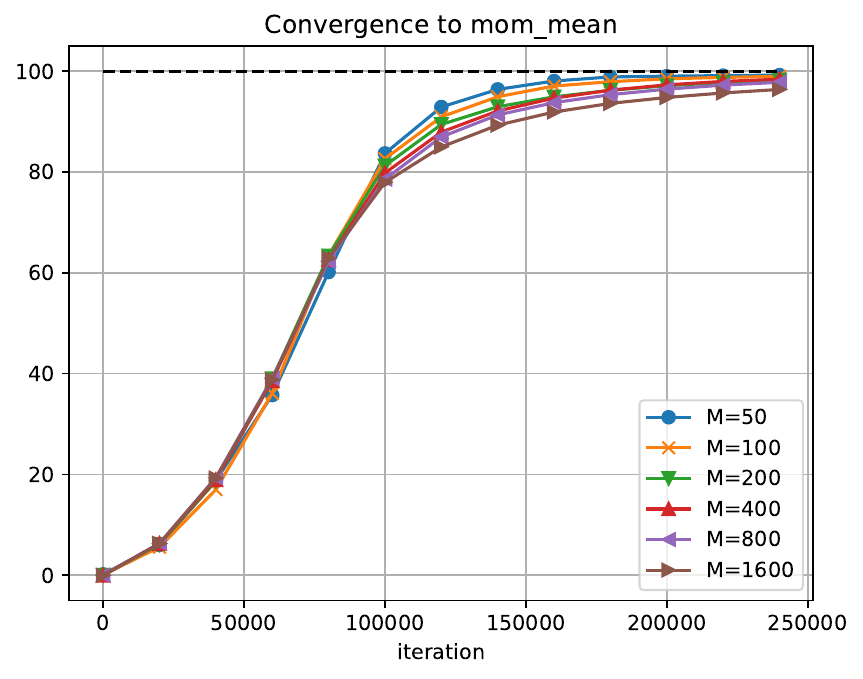}\caption{Convergence of $(\sum_{m=1}^{M}w_{m})^{-1}\sum_{m=1}^{N}w_{m}r_{m}$ to $N^{-1}\sum_{n=1}^{N}s_{n}$ where $s_{n}$ is the $n$-th mother test score.}
	\label{fig:mom_mean}
\end{figure}

\begin{figure}
	\includegraphics[width=\linewidth]{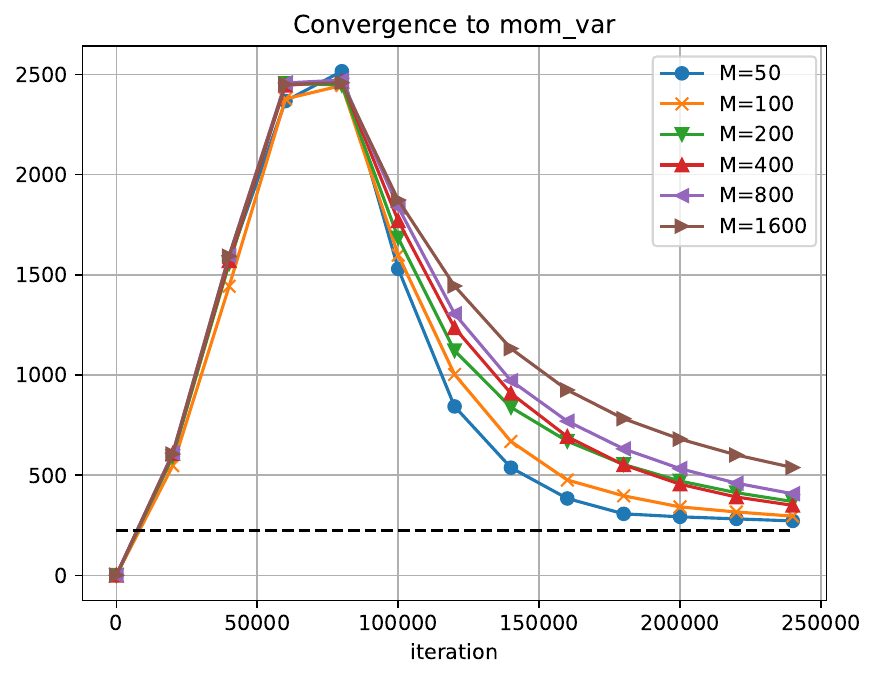}\caption{Convergence of $(\sum_{m=1}^{M}w_{m})^{-1}\sum_{m=1}^{N}w_{m}r_{m}^{2} - \left((\sum_{m=1}^{M}w_{m})^{-1}\sum_{m=1}^{M}w_{m}r_{m}\right)^{2}$ to $N^{-1}\sum_{n=1}^{N}s_{n}^{2} - \left( N^{-1}\sum_{n=1}^{N}s_{n}\right)^{2}$ where $s_{n}$ is the $n$-th mother test score.}
	\label{fig:mom_var}
\end{figure}

\begin{figure}
	\includegraphics[width=\linewidth]{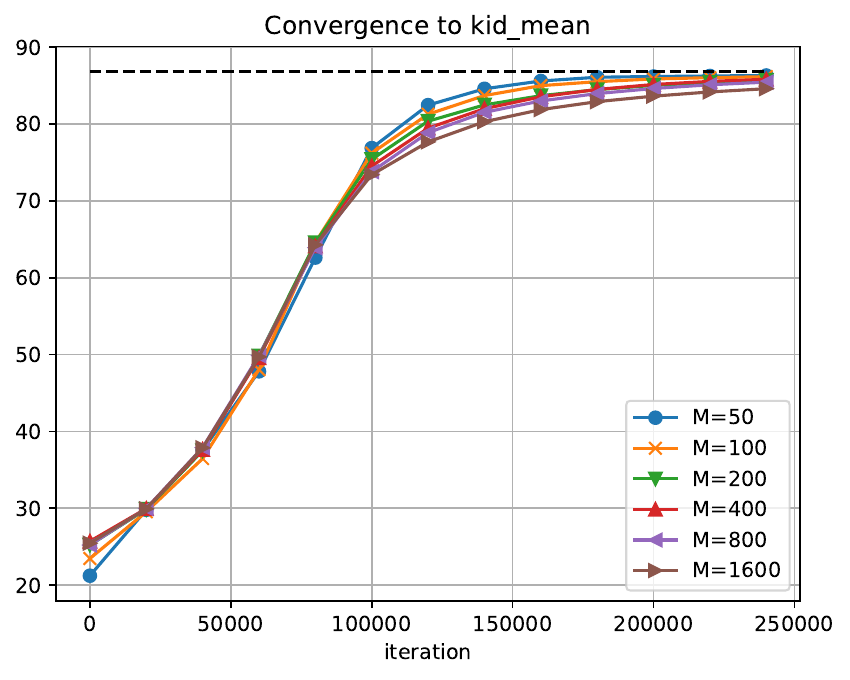}\caption{Convergence of $(\sum_{m=1}^{M}w_{m})^{-1}\sum_{m=1}^{N}w_{m}u_{m}$ to $N^{-1}\sum_{n=1}^{N}y_{n}$ where $y_{n}$ is the $n$-th kid test score.}
	\label{fig:kid_mean}
\end{figure}

\begin{figure}
	\includegraphics[width=\linewidth]{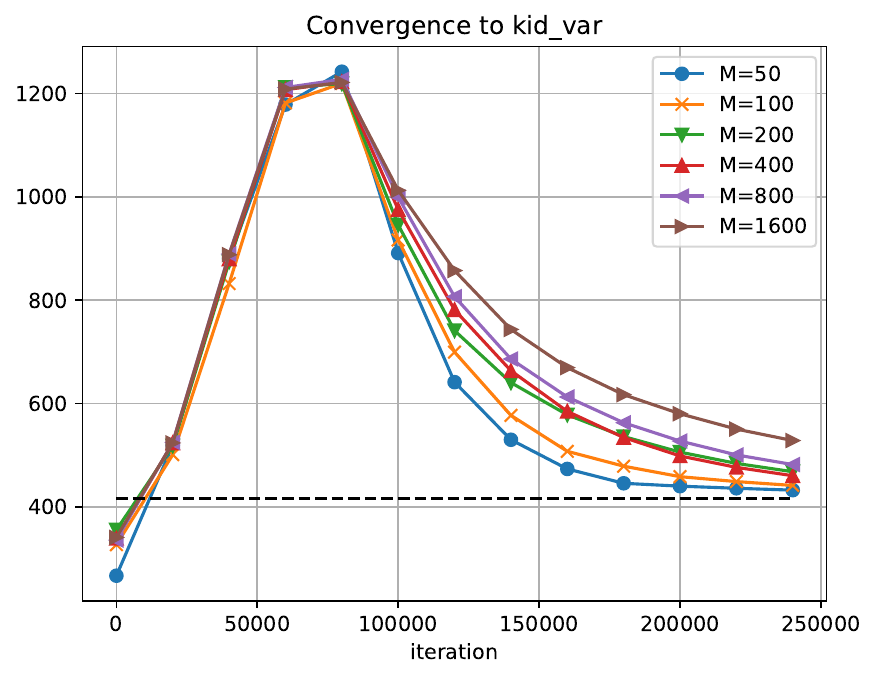}\caption{Convergence of $(\sum_{m=1}^{M}w_{m})^{-1}\sum_{m=1}^{N}w_{m}u_{m}^{2} - \left((\sum_{m=1}^{M}w_{m})^{-1}\sum_{m=1}^{M}w_{m}u_{m}\right)^{2}$ to $N^{-1}\sum_{n=1}^{N}y_{n}^{2} - \left( N^{-1}\sum_{n=1}^{N}y_{n}\right)^{2}$ where $y_{n}$ is the $n$-th kid test score.}
	\label{fig:kid_var}
\end{figure}

The figures below show the convergence of the reconstructed weights and data to the statistics of the target model that were shown to be vulnerable by Lemma \ref{lem:reco_vals}. Figure \ref{fig:N} shows convergence of the sum of the weights to the total number of training data points. Figure \ref{fig:mom_mean} and Figure \ref{fig:mom_var} show (respectively) the convergence of the empirical, weighted mean (respectively variance) of the reconstructed mom scores to the empirical mean (respectively variance) of the true, unknown mom scores. Figure \ref{fig:kid_mean} and Figure \ref{fig:kid_var} show (respectively) the convergence of the empirical, weighted mean (respectively variance) of the reconstructed child scores to the empirical mean (respectively variance) of the true, unknown child scores.

This shows that even though full, exact training data reconstruction is not possible, the theoretical chracterisation and numerical objective presented in previous sections provides an adversary the ability to gain other information about the unknown training data. In this case an understanding of number of children used in the model, the mean and variance of their scores as well as the mean and variance of the scores of their mothers scores. This could then be used by an adversary to bootstrap other inferences about the data, such as the ages of the children or mothers.

\section{Conclusion and Future Directions}\label{sec:conclusion}
This paper has provided a concrete mathematical framework to analyze the data reconstruction problem by expressing the problem purely in terms of empirical measures of the training data. For the first time in the literature the Bayesian case has been covered and the non-Bayesian setting follows as a simple, natural consequence which recovers existing methods in the literature. This shows that data reconstruction attacks for both the Bayesian and non-Bayesian setting can be viewed as score-based problems. 

The reconstruction method and characterisation result begs many further questions. A full investigation into the numeric properties of the reconstruction methods using more advanced score-matching methods, such as de-noising methods, would be valuable. Using the MMD representation of the possible features that can be reconstructed also has great potential. This could be used to quantitatively evaluate how susceptible a model is to reconstruction attacks by evaluating how characteristic the corresponding kernel based on the model features is. Another question along this line is understanding quantitatively the trade off between model complexity making reconstruction easier and increasing the number of training points making reconstruction harder. 

A topic not studied in this paper is differential privacy. Current guarantees of differential privacy to stop reconstruction attack focus on very different assumptions than those used in this paper, for example the concept of Reconstruction Robustness, often abbreviated to ReRo, defined by \citet{Balle2022_informed} has become a popular item of study to combat reconstruction attacks and assumes the attacker has all but one of the data points. As has been shown in the numerics section, sufficient statistics of models can be reconstructed. Therefore, focus on the sufficient statistic pertubation methods would be natural \citep{Bernstein2019,Alabi2022}. 

Finally, the optimisation methods used in Section \ref{sec:numerics} are somewhat basic, simply gradient descent on the objective. Further analysis of the optimisation, using the representation of the objective as a MMD, is needed. 
\bibliography{example_paper}
\bibliographystyle{icml2022}

\newpage
\appendix
\onecolumn
\section{Sliced Fisher Divergence}
This section shall cover the simple adaptation from Fisher divergence to sliced Fisher divergence (SFD) \citep{song2019sliced} for the results in Section \ref{sec:Bayesian}. Sliced Fisher divergence has the following form for a standard normal slicing distribution $p_{\mathcal{N}}$ over $\bbR^{d}$ where $\theta\in\bbR^{d}$
\begin{align}\label{eq:SFD}
	\text{SFD}(\pi_{\bX},\pi_{w,Z}) = \frac{1}{2}\int_{\bbR^{d}}\int_{\Theta}\langle v,\nabla_{\theta}\log\pi_{\bX}(\theta)\rangle_{\bbR^{d}} - \langle v,\nabla_{\theta}\log\pi_{w,\bZ}(\theta)\rangle_{\bbR^{d}}^{2}\mathrm{d}\pi_{\bX}(\theta)\mathrm{d}p_{\mathcal{N}}(v).
\end{align}

The idea of sliced fisher divergence is that it retains the desirable properties of the Fisher divergence while being more computationally viable. This is due to the following integration-by-parts rearrangement, an analogous result to Equation \eqref{eq:FD_IBP}, being $O(1)$ rather than $O(d)$ to estimate,
\begin{align}\label{eq:SFD_IBP}
	\text{SFD}(\pi_{\bX},\pi_{w,\bZ}) = \bbE_{p\times\pi_{\bX}}[\langle v,\nabla_{\theta}^{2}\log\pi_{w,\bZ}(\theta)v\rangle_{\bbR^{d}}] + \frac{1}{2}\bbE_{\pi_{\bX}}[\lVert \nabla_{\theta}\log\pi_{w,\bZ}(\theta)\rVert_{\bbR^{d}}^{2}] + C',
\end{align}
where $C'$ is a constant that doesn't depend on $w,\bZ$.

Given samples $\{\theta_{t}\}_{t=1}^{T}$ from $\pi_{\bX}$ and $\{v_{tl}\}_{t=1,l=1}^{T,L}$ for some user chosen $L$ the SFD in Equation \eqref{eq:SFD_IBP} can be estimated as
\begin{align}\label{eq:SFD_est}
	\text{SFD}(\pi_{\bX},\pi_{w,\bZ}) \approx \frac{1}{TL}\sum_{t=1}^{T}\sum_{l=1}^{L}\langle v_{tl},\nabla_{\theta}^{2}\log\pi_{w,\bZ}(\theta_{t})v_{tl}\rangle_{\bbR^{d}} + \frac{1}{2T}\sum_{t=1}^{T}\lVert\nabla_{\theta}\log\pi_{w,\bZ}(\theta_{t})\rVert^{2},
\end{align}
for the derivation see \citet{song2019sliced}. Choices other than standard normal for the slicing distribution are available.

Armed with the definition of SFD a result analogous to Theorem \ref{thm:characterise} can be easily derived.

\begin{proposition}\label{prop:SFD_MMD}
	Let $\bX = \{\bx_{n}\}_{n=1}^{N}, \bZ = \{\bz_{m}\}_{m=1}^{M}$ be two data sets and $w = \{w_{m}\}_{m=1}^{M}$ a set of scalar weights. Let $\pi_{0}$ be a prior for an unknown parameter lying in $\Theta \subset \bbR^{d}$, $l$ be a likelihood function and $p_{\mathcal{N}}$ the standard normal on $\bbR^{d}$. For  ${\pi_{\bX}(\theta)\propto\prod_{n=1}^{N}l(\theta,\bx_{n})\cdot\pi_{0}(\theta)}$ and ${\pi_{w,\bZ}(\theta)\propto\prod_{m=1}^{M}l(\theta,\bz_{m})^{w_{m}}\cdot\pi_{0}(\theta)}$  
	\begin{align*}
		\emph{SFD}(\pi_{\bX},\pi_{w,\bZ}) = \frac{1}{2}\:\emph{MMD}_{k}(P_{\bX}, P_{w,\bZ})^{2} 
	\end{align*}
	where 
	\begin{align*}
		k(\bx,\bx') = \int_{\bbR^{d}}\int_{\Theta}\langle v,\nabla_{\theta}\log l(\theta,\bx)\rangle_{\bbR^{d}}\langle v,\nabla_{\theta}\log l(\theta,\bx')\rangle_{\bbR^{d}} \mathrm{d}\pi_{\bX}(\theta)\mathrm{d}p_{\mathcal{N}}(v)
	\end{align*}
	and $P_{\bX} = \sum_{n=1}^{N}\delta_{\bx_{n}}, P_{w,\bZ} = \sum_{m=1}^{M}w_{m}\delta_{\bz_{m}}$ are the un-normalised empirical data measures.
\end{proposition}
The proof is a simple adaptation of the proof of Theorem \ref{thm:characterise} by doing a term by term comparison of Equation \ref{eq:SFD} and the defintion of MMD.

In fact, the choice of $p_{\mathcal{N}}$ being standard multivariate normal for the slicing distribution allows for an exact equivalance to FD. 

\begin{lemma}\label{lem:SFD_FD}
	Given the assumptions of Proposition \ref{prop:SFD_MMD} $\emph{SFD}(\pi_{\bX},\pi_{w,\bZ}) = \emph{FD}(\pi_{X},\pi_{w,\bZ})$ and the kernels for their corresponding MMD expressions, from Theorem \ref{thm:characterise} and Proposition \ref{prop:SFD_MMD}, respectively, are equal.
\end{lemma}
\begin{proof}
	This result is a simple consequence of the fact that if $v$ is a standard multivariate Gaussian in $\bbR^{d}$ and $u\in\bbR^{d}$ then $\langle v,u\rangle_{\bbR^{d}}\sim N(0,\lVert u\rVert_{\bbR^{d}}^{2})$. Substituting $u = \nabla_{\theta}\log\pi_{\bX}(\theta) - \nabla_{\theta}\log\pi_{w,\bZ}(\theta)$ then using Equation \eqref{eq:FD}, Equation \eqref{eq:SFD} and re-arranging integrals compeletes the result for equivalance of SFD and FD. For the equivalance of the kernels use the fact that if $v$ is a standard multivariate Gaussian in $\bbR^{d}$ and $u,w\in\bbR^{d}$ then $\bbE[\langle v,u\rangle_{\bbR^{d}}\langle v,w\rangle_{\bbR^{d}}] = \langle u,w\rangle_{\bbR{^d}}$, substitute $u = \nabla_{\theta}\log l(\theta,\bx), w = \nabla_{\theta}\log l(\theta,\bx')$ and use the definitions of the kernels in Theorem \ref{thm:characterise} and Proposition \ref{prop:SFD_MMD} to complete the proof. 
\end{proof}

When using SFD to perform DRA, at each iteration the user will draw $\{v_{tl}\}_{t=1,l=1}^{T,L}$ from $p_{\mathcal{N}}$ and use the estimate in Equation \eqref{eq:SFD_est} and then use auto-diff to take a gradient step with respect to $w,\bZ$. 

\section{Proof of Lemma \ref{lem:reco_vals}}
The proof sttrategy is very similar to what is done in Example \ref{exp:BLR}. First, the feature maps of the kernel are written out. Then, using the equivalence between MMD and FD, one can deduce that if the FD is zero then the difference between the features maps in the feature space is zero. As the feature space is $L^{2}(\pi_{\bX})$ this means that the feature maps must be equal as functions of $\theta$ whereever $\pi_{\bX}$ has support. In the numerical example used in Section \ref{sec:numerics} $\pi_{\bX}$ has full support. Therefore, we can use particular values of $\theta$ to identify what values the weights and psudeo-data must take. 

Since $\theta = (\beta,\sigma)$, the feature map at one point is $\varphi(\bx)$ which is the function $\varphi(\bx)(\theta) = \nabla_{\theta}\log l(\theta,\bx) = (\nabla_{\beta}\log l(\theta,\bx),\nabla_{\sigma}\log l(\theta,\bx))\eqqcolon (\varphi_{\beta}(\bx)(\theta),\varphi_{\sigma}(\bx)(\theta))$. So in particular, each of these two entries must be equal as functions of $\theta$ under the expectations of $P_{\bX}$ and $P_{w,\bZ}$. For the first term, this means the following two expressions are equal as functions of $\beta,\sigma$
\begin{align*}
	\bbE_{P_{\bX}}[\varphi_{\beta}(\bx)(\theta)] & = \bbE_{P_{\bX}}\left[-\frac{1}{\sigma^{2}}x(\langle \beta, x\rangle - y)\right] = -\frac{1}{2\sigma^{2}}\sum_{n=1}^{N}x_{n}x_{n}^{\top}\beta - x_{n}y_{n}, \\
	\bbE_{P_{w,\bZ}}[\varphi_{\beta}(\bx)(\theta)] & = \bbE_{P_{w,\bZ}}\left[-\frac{1}{\sigma^{2}}x(\langle \beta, x\rangle - y)\right] = -\frac{1}{2\sigma^{2}}\sum_{m=1}^{M}w_{m}z_{m}z_{m}^{\top}\beta - z_{m}u_{m}. \\
\end{align*}
From this we can conclude the first two equivalances in Lemma \ref{lem:reco_vals}. For the third equivalence, we compare the features in $\varphi_{\sigma}$
\begin{align*}
	\bbE_{P_{\bX}}[\varphi_{\sigma}(\bx)(\theta)] & = \bbE_{P_{\bX}}\left[\frac{1}{\sigma^{3}}(\langle \beta, x\rangle - y)^{2} - \frac{1}{\sigma}\right] = \frac{1}{\sigma^{3}}\sum_{n=1}^{N}(\langle \beta, x_{n}\rangle - y_{n})^{2} + \frac{N}{\sigma}, \\
	\bbE_{P_{w,\bZ}}[\varphi_{\sigma}(\bx)(\theta)] & = \bbE_{P_{w,\bZ}}\left[\frac{1}{\sigma^{3}}(\langle \beta, x\rangle - y)^{2} - \frac{1}{\sigma}\right] = \frac{1}{\sigma^{3}}\sum_{m=1}^{M}w_{m}(\langle \beta, z_{m}\rangle - u_{m})^{2} + \frac{\sum_{m=1}^{M}w_{m}}{\sigma}.
\end{align*}
The fact that $N=\sum_{m=1}^{M}w_{m}$ is established by the derivations above which means the second terms in the two expressions above are equal. Since this equality holds when intergrating over all of $\pi_{\bX}$ it means it holds in particular over a ball around $\beta = 0$. This means that $\frac{1}{\sigma^{3}}\sum_{n=1}^{N}y_{n}^{2} = \frac{1}{\sigma^{3}}\sum_{m=1}^{M}w_{m}u_{m}^{2}$ for all $\sigma$ in the support of $\pi_{\bX}$, which is fully supported and so the third equivalance in Lemma \ref{lem:reco_vals} is proved.

\end{document}